\documentclass[journal]{IEEEtran}
\ifCLASSINFOpdf
\else
\fi

\usepackage{amsmath,amsthm,amssymb}
\usepackage{algorithm}
\usepackage{algorithmic}
\usepackage{color}

\usepackage{comment}
\usepackage{wrapfig}
\usepackage{multirow}

\usepackage{amsthm}
\usepackage{graphicx}
\usepackage{subfigure}
\newtheorem{proposition}{Proposition}

\begin{document}
%
\title{MetaMixUp: Learning Adaptive Interpolation Policy of MixUp with Meta-Learning}
%
%
%
\author{Zhijun Mai, Guosheng Hu, Dexiong Chen, Fumin Shen, Heng Tao Shen}

\maketitle

\begin{abstract}
MixUp is an effective data augmentation method to regularize deep neural networks via random linear interpolations between pairs of samples and their labels. It plays an important role in model regularization, semi-supervised learning and domain adaption.
However, despite its empirical success, its deficiency of randomly mixing samples has poorly been studied. Since deep networks are capable of memorizing the entire dataset, the corrupted samples generated by vanilla MixUp with a badly chosen interpolation policy will degrade the performance of networks. To overcome the underfitting by corrupted samples, inspired by Meta-learning (learning to learn), we propose a novel technique of learning to mixup in this work, namely, MetaMixUp.
Unlike the vanilla MixUp that samples interpolation policy from a predefined distribution, this paper introduces a meta-learning based online optimization approach to dynamically learn the interpolation policy in a data-adaptive way. The validation set performance via meta-learning captures the underfitting issue, which provides more information to refine interpolation policy. Furthermore, we adapt our method for pseudo-label based semi-supervised learning (SSL) along with a refined pseudo-labeling strategy. In our experiments, our method achieves better performance than vanilla MixUp and its variants
under supervised learning configuration. In particular, extensive experiments show that our MetaMixUp adapted SSL greatly outperforms MixUp and many state-of-the-art  methods 
on CIFAR-10 and SVHN benchmarks under SSL configuration.
\end{abstract}

\begin{IEEEkeywords}
Deep Learning, MixUp, Meta-learning, Regularization.
\end{IEEEkeywords}

%
\IEEEpeerreviewmaketitle

\section{Introduction}
%
%
%
%
\IEEEPARstart{D}{espite} their striking success in many challenging  tasks, deep neural networks have shown prone to overfitting, especially when the number of annotated samples is scarce as in the weakly-supervised \cite{OquabBLS15,DurandMTC17}  and semi-supervised learning \cite{miyato2018virtual,TarvainenV17}. In addition, this is also reflected in high generalization errors when the ability of deep CNNs overfits or memorizes the corrupted samples that have slight distributional shifts, which are also known as the imperceptible adversarial perturbations \cite{SzegedyZSBEGF13}. These issues can make deep learning based systems degrade the prediction performance in practice. It is thus desirable to design effective regularization methods to control the model complexity and reduce the gap between training error and generalization error.  

Recently due to the great advances of machine learning, remarkable regularization methods have been proposed. In addition to manual designed regularization architecture network Shake-Shake Regularization \cite{Gastaldi17}, typically adding noise to the deep model is important to reduce overfitting and learn more robust abstractions, e.g., dropout \cite{SrivastavaHKSS14} and randomized data augmentation \cite{KrizhevskySH17}. A simple and effective method, called MixUp~\cite{zhangmixup}, has been proposed recently as a data augmentation scheme to address generalization problems. Specifically, it is performed to generate additional virtual samples during training via a simple linear interpolation of randomly picked training sample pairs, as well as their labels. However, its interpolation policy (the weights for interpolating  paired samples) is randomly chosen from some prior distribution (e.g. Beta Distribution) for each pair of samples at each iteration, which may lead to manifold intrusion and thus underfitting~\cite{Adamixup}. We observe that the original MixUp is not robust in cases when the generated virtual samples are adjacent to real samples of different categories, the corresponding virtual labels becomes ambiguous. Nevertheless, original MixUp method does not take such ambiguities into account. 
Therefore, carefully choosing an adequate interpolation policy to avoid underfitting is crucial to achieve promising performance.
It is not trivial to learn better interpolation policy for MixUp technique since deep CNNs are more prone to memorize corrupted samples and improving deep CNNs on corrupted samples and labels is clearly an under-studied problem and worthy of exploration. AdaMixup\cite{Adamixup} proposes to assess the quality of interpolation policy with well selected triplet data and uses an additional intrusion discriminator to judge whether the sample generated by a policy collides with a real data point. Nevertheless, relying on training an additional carefully designed network that estimates the interpolation as a supervision signal, their method has additional hyperparameters (e.g. triplets selection, larger model architecture, more optimization parameters, etc.) to tune and can be hard to deploy for a new dataset or task.

In this paper, we also propose a new theoretical perspective of MixUp by showing its empirical risk as a lower bound of the gradient Lipschitz constant of a neural network. This observation not only helps understanding vanilla MixUp better but also puts forward again the underfitting issue caused by naive choice of the interpolation policy. Our method is inspired by recent success of Meta-Learning, a learning paradigm inspired by the cognitive process of human and animals, in which a model is learning to learn better using validation set as meta-data. This paper tackles manifold intrusion/underfitting issue using meta-learning on MixUp to learn the interpolation policy in a data-adaptive way. Meta-learning has successfully shown to be powerful in learning data-adaptive rules and policies, such as initial neural network weights \cite{finn2017model}, optimization hyperparameters \cite{snoek2012practical}, unsupervised learning rules \cite{metz2018learning} etc., making models more general and adaptive to the new datasets and tasks. For our problem, our intuition is that a meta model with random interpolation policy learns to gained knowledge from metadata, can provide instructive supervision for vanilla MixUp to refine interpolation policy in a data-driven style. A reasonable interpolation policy for MixUp can help deep CNNs alleviate manifold intrusion problem made by corrupted labels and samples.
Our method, dubbed MetaMixUp, consists of learning the interpolation policy of MixUp by adapting a meta-learning method in a data-adaptive way. Specifically, we aim to learn a interpolation policy to minimize the expected loss for the training set. Meta model can be learned to discover new data-driven interpolation policy from metadata. The learned data-driven interpolation policy can be updated a few times taking into account of the main model's feedback. Whenever the interpolation policy is learned, we turn the deep CNNs from meta-stage to main-stage to minimize the learning objective, where the main-stage controls training procedure to learn each mixed sample. At the test time, deep CNNs makes predictions alone in main-stage. Instead of searching a discrete set of candidate interpolation policy, we relax the optimization via an online gradient-based meta-learning to make it continuous, so that the interpolation policy can be optimized with respect to its validation set performance by gradient descent. The data-adaptive of gradient-based optimization, as opposed to selection from prior distribution, allows MetaMixUp to achieve competitive even better performance for different tasks.

To our best knowledge, our method is the first one that applies meta-learning to guide interpolation policy learning for MixUp technique. It tackles the manifold intrusion problem in a more direct and simple way and leads to better performance over original MixUp and recent proposed AdaMixUp on typical image classification benchmarks: ImageNet, MNIST, SVHN, Fashion-MNIST, CIFAR-10 and CIFAR-100 under supervised configuration. To demonstrate its adaption for semi-supervised tasks, our method extends MixUp to the pseudo-label based methods~\cite{lee2013pseudo} and further adopts an asynchronous pseudo labeling strategy. Our resulting semi-supervised method improves the performance over the original pseudo-label based SSL method by a large margin and achieves comparable performance over state-of-the-art methods on CIFAR-10 and SVHN under semi-supervised configuration. Furthermore, we apply MetaMixUp to a powerful MixUp augmented SSL method called MixMatch \cite{mixmatch}, and improve the previous state-of-the-art results, which suggests that our MetaMixUp is auxiliary to other methods of semi-supervised learning.

To sum up, we highlight our threefold contributions as follows. 
\begin{enumerate}
\item We address underfitting issue caused by badly chosen interpolation policy of vanilla MixUp. And we introduce a new theoretical perspective that MixUp is a lower bound of the Lipschitz constant of the gradient of the neural network to help further understanding vanilla MixUp. 
\item We propose a gradient-based meta-learning algorithm which is exploited to guide refining interpolation policy of MixUp in a data-driven way. The model can be optimized with respect to its validation set performance by gradient descent. We relax the optimization with an online approximation to improve training efficient. We find that MetaMixUp achieves the best performance over vanilla MixUp and AdaMixup.
\item We extend our MetaMixUp and MixUp to semi-supervised learning tasks with an asynchronous pseudo labeling strategy. Through extensive experiment we show that our extensions achieve highly competitive results on CIFAR-10 and SVHN, which we attribute to their adaption for other tasks.
\end{enumerate}

The rest of this paper is organized as follows. In section \uppercase\expandafter{\romannumeral2}, we review the literature relevant to our work. In section \uppercase\expandafter{\romannumeral3}, a new perspective of MixUp is introduced, and the proposed MetaMixUp along with its extensions to SSL are presented in detail, respectively. We provide experimental results and analysis in section \uppercase\expandafter{\romannumeral4}, and summarize this paper in section \uppercase\expandafter{\romannumeral5}.

 

\section{Related Work}
\subsection{Regularization}
Regularization is an ongoing subject in machine learning and has been widely studied. It refers to the general approach of penalizing the amount of information neural network contains to keep the parameters simple \cite{HintonC93}. The constraints and disturbance on model keep it from over-fitting to training data and thus hopefully make it generalize better to test data. In particular, a common regularization technique is to add a loss term which penalizes the L2 norm of the model parameters. When we are using simple gradient descent optimizer such as Adam \cite{KingmaB14}, this loss term is equivalent to weight decay, which exponentially decaying the weight values towards zero in training procedure.
Data augmentation techniques are commonly used regularizer by leveraging additional samples generated by appropriate domain-specific transformations. For instance, random cropping, flipping and rotating are typical data-augmentation ways for image data~\cite{KrizhevskySH17,HeZRS16}. Dropout is another very helpful regularizer in avoiding over-fitting by randomly dropping units from the neural network during training \cite{SrivastavaHKSS14}. In contrast to these data-independent methods, AutoAugment~\cite{cubuk2018autoaugment} has proposed a data-adaptive way to search the best data-augmentation policy from a huge space of policies, which are combinations of many sub-policies. On the other hand, instead of operating on single image sample, MixUp~\cite{zhangmixup} and between-class learning~\cite{TokozumeUH18}
augment training data points by interpolating multiples examples and labels. Manifold Mixup\cite{Vikas18}  leverages semantic interpolations in random layers as additional training signal to train neural networks. Nonetheless, their interpolation policies are predefined
and are not data-driven. Our approach is a data-driven extension of MixUp via meta-learning, which leverages vicinal relations between examples and can alleviate manifold intrusion problem introduced in \cite{Adamixup}. It is closely related to AdaMixup~\cite{Adamixup}, which also learns the mixing policy from data. While AdaMixup requires training an additional network to infer the policy and also 
an intrusion discriminator {with well selected triplets}, our method is directly applied to the original MixUp method without adding further components (e.g. a carefully designed discriminator) to the model.

\subsection{Meta-learning}
 Meta-learning methods date back to the 90s \cite{thrun1998learning,bengio1990learning} and have recently resurged with various techniques focused on learning how to learn and thus quickly adapt to new information \cite{RaviL17}. Meta-learning approaches can be broadly categorized into three groups, which has been proposed to solve the few-shot learning problem. \textit{Gradient-based methods} \cite{finn2017model,RaviL17} use gradient descent to adapt the model parameters. \textit{Nearest-neighbor methods} \cite{SnellSZ17} learn prediction rule over the embeddings based on distance to the nearest class mean. \textit{Neurons-based methods} \cite{MishraR0A18,MunkhdalaiYMT18} learn meta procedure of how to adapt the connections between neurons for different tasks. Our method is tightly related to gradient-based meta-learning algorithm MAML \cite{finn2017model}. MetaMixUp also implicitly learns how to quickly adapt to new datasets and tasks through a gradient-based meta-learning algorithm. Unlike MAML, our optimization procedure works in an online fashion rather than
relying on heavy offline training stages. Similar to Meta-learning, recent hashing researches focused on learning to hash \cite{ShenSHT13}. Their goal is to learn data-dependent hash functions which generate more compact codes to achieve good search accuracy \cite{ShenGLYS17, ShenSLS15}. This paper is more similar to the optimization-based meta-learning method for sample reweighting \cite{RenZYU18}, 
which has focused on imbalanced classification and noisy label problems.

\subsection{Hyperparameter optimization}
Performance of machine learning algorithms depends critically on identifying a good set of hyperparameters. Recent interest in complex and computationally expensive machine learning models with many hyperparameters, such as automated machine learning (AutoML) frameworks and deep neural networks, has resulted in a resurgence of research on hyperparameter optimization (HPO). The current gold standard methods for hyperparameter selection are blackbox optimization methods. Due to the non-convex nature of the problem, global optimization algotithms are usually applied. The standard baseline methods involve training tens of models that select hyperparameters configurations randomly and nonadptively\cite{BergstraYC13,ThorntonHHL13}, e.g., grid or random search. Moreover, the majority of recent work in this growing area focuses on Bayesian hyperparameter optimization \cite{SnoekLA12} with the goal of optimizing hyperparameter configuration selection in an iterative fashion. However, recent gradient-based techniques for HPO have significantly increase the number of hyperparameters that can be optimized \cite{Pedregosa16}. In this way, it is now possible to tune large-scale weight vectors as hyperparameters associated with neural networks. Such an approach is suited for MixUp technique where the interpolation weight of each sample pair are treated as hyperparameters across a set of training episodes.

\subsection{Semi-Supervised learning}
Semi-Supervised learning has been extensively studied and has a large variety of groups. Typical successful SSL methods have involved some consistency regularization such as $\Pi$-model~\cite{LaineA16}, VAT~\cite{miyato2018virtual}, Mean Teacher \cite{TarvainenV17} and simple label propagation such as pseudo-labeling~\cite{lee2013pseudo} or more generally self-training~\cite{de1994learning,SaitoUH17}. Our method is more similar to pseudo-labeling based methods. Basically, pseudo-labels are the current predictions of the classifier assigned to unlabeled examples. \cite{SaitoUH17} proposed to leverage multiple networks to asymmetrically give pseudo-labels to unlabeled samples. \cite{XieZCC18} proposed moving average centroid alignment to reduce bias caused by false pseudo-labels. However, these methods ignore considering the stability of the pseudo-labels. We extend MixUp and our MetaMixUp to SSL tasks with an asynchronous pseudo-labeling strategy to stabilize the training. Another recent proposed MixMatch \cite{mixmatch} works by guessing low-entropy labels for unlabeled examples under multiple data augmentations for each sample. Likewise, it combines MixUp with consistency regularization.

\section{Data-adaptive MixUp via Meta Learning}
In this section, we first introduce a new perspective of MixUp. We then detail our algorithm of learning the interpolation policy for MixUp. Finally, we adapt our proposed MetaMixUp for supervised and semi-supervised tasks respectively.

\begin{figure*}[h]
	\centering  
	\includegraphics[width=0.9\textwidth,height=0.5\textwidth]{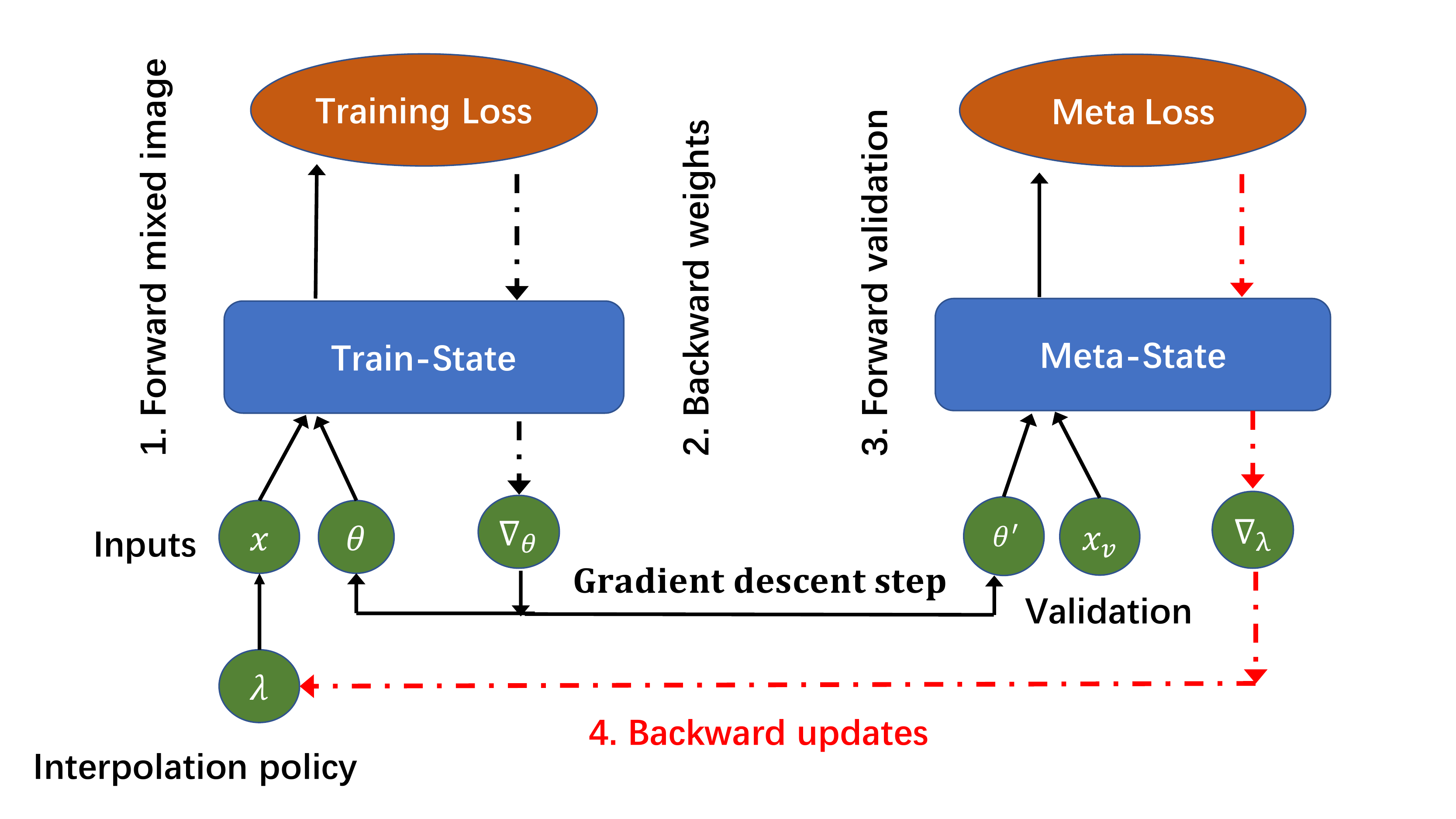}
	\caption{Computation Graph of Our MetaMixUp}  
	\label{fig:figure1}
\end{figure*}

\subsection{MixUp as A Lower Bound of the Gradient Lipschitz Constant}
MixUp, originally proposed by \cite{zhangmixup}, augments the training set 
by linearly interpolating a random pair of examples and their corresponding 
labels selected in a minibatch through permutation: 
\begin{equation}\label{eq:mixup}
\begin{split}
	\Tilde{ x}_{i} &= \lambda  x_{i} + (1-\lambda)  x_{j}, \\
	\Tilde{ y}_{i} &= \lambda y_{i} + (1-\lambda)  y_{j},
\end{split}
\end{equation}
where $(x_{i}, y_{i})$ and $(x_{j},y_{j})$ are two data-target samples randomly drawn 
from the training set, and $\lambda \in [0, 1]$ {is the interpolation weighting coefficient.} Then the objective of a supervised problem becomes minimizing the empirical risk over the MixUp-generated samples.

Despite the empirical effectiveness of MixUp, how it controls the smoothness of a neural network has hardly been investigated. As discussed in \cite{Adamixup}, bad interpolation coefficient $\lambda$ can lead to underfitting caused by the manifold intrusion. This occurs when 
the MixUp-generated sample collides with an existing real example with
a label different from the interpolated pair's, and thus leads to performance degradation. Here, we consider MixUp from a regularization point of view and show that it is a lower bound of the Lipschitz constant of the gradient of the neural network.

While many previous works have controlled the smoothness of a neural network by controlling its Lipschitz constant \cite{cisse2017parseval,TsuzukuSS18,miyato2018virtual}, we consider here to control a stronger condition, which is the Lipschitz constant of its (sub)gradient. Specifically, we assume that the predictive function $f:\mathbb{R}^d\to\mathbb R$ is a differentiable function and its gradient is $\kappa$-Lipschitz continuous:
\begin{equation}\tag{P1}\label{eq:grad_lip}
 \forall{x, x' \in \mathbb{R}^d} \quad	\| \nabla f(x)-\nabla f(x')\| \le \kappa\| x - x'\|
\end{equation}
On the other hand, we consider the following inequality: 
\begin{equation}\tag{P2}\label{eq:mixup_bound}
\begin{split}
	|f(\lambda x + (1 - \lambda)x' ) - [\lambda f(x) + (1 - \lambda) f(x')]| \\
	\le \frac{\lambda (1-\lambda )\kappa}{2}\| x - x' \|^2
\end{split}
\end{equation}
where $\lambda\in [0,1]$. {Under MixUp setting, $x$ and $x'$ in (P2) can represent $x_i$ and $x_j$ in \eqref{eq:mixup}.} Noting that the left term of (P2) is equivalent to the empirical risk of MixUp by replacing the $\ell1$-loss with a general loss function and the prediction $f(x)$ and $f(x')$ with their true labels. 
MixUp loss can be considered as a proxy of the left term.

Now we announce the following proposition which builds the relation between the Lipschitz continuity of the gradient and the empirical risk of MixUp.
\begin{proposition}[Link between MixUp and gradient Lipschitz continuity]
	The property \eqref{eq:grad_lip} $\Rightarrow$ \eqref{eq:mixup_bound}.
\end{proposition}


\begin{proof}
For all $x$ and $x'$ in $\mathbb R^d$ we have
\vspace{-3mm}
\begin{small}
	\begin{equation}
		\begin{split}
			& f(\lambda x + (1-\lambda )x' )\\ 
			=&f(x') + \lambda \int_0^1 \left\langle \nabla f(\lambda t x + (1-\lambda t)x'), x-x' \right\rangle dt \\
			=&f(x')+\lambda [f(x) - f(x')] \\
			+& \lambda \left[\int_0^1 \left\langle \nabla f(\lambda t x + (1-\lambda t)x'), x-x' \right\rangle dt - (f(x)-f(x'))\right] 
		\end{split}
	\end{equation}
\end{small}
Therefore
\begin{small}
\begin{equation}
	\begin{split}
		& |f(\lambda x + (1-\lambda )x' ) - (\lambda f(x) + (1 - \lambda )f(x')|\\
		=& \lambda \left|\int_0^1 \left\langle \nabla f(\lambda t x + (1-\lambda t)x') - \nabla f(t x + (1 - t)x'), x-x' \right\rangle dt\right| \\
		\le & \lambda \int_0^1 \left| \left\langle \nabla f(\lambda t x + (1-\lambda t)x') - \nabla f(t x + (1 - t)x'), x-x' \right\rangle \right| dt \\
		\le & \lambda \int_0^1 \| \nabla f(\lambda t x + (1-\lambda t)x') - \nabla f(t x + (1 - t)x')\| \| x-x' \| dt  \\
		\le & \lambda \int_0^1 (1 - \lambda )t \kappa\| x - x'\|^2 dt \\
		=& \frac{\lambda(1-\lambda )\kappa}{2} \| x - x' \|^2,
	\end{split}
\end{equation}
\end{small}
where the second inequality follows Cauchy-Schwarz inequality and the third inequality is the property $P1$.

\end{proof}

This proposition suggests controlling the Lipschitz constant of the gradient necessarily requires minimizing MixUp loss. However when $x$ is far from $x'$, the mixing policy $\lambda$ has a much greater effect on the Lipschitz constant. Therefore, minimizing the MixUp loss with a badly chosen interpolation policy $\lambda$ cannot help control the Lipschitz constant but leads to unexpected degradation. 
This observation shows the importance of elaborating a smarter way to choose $\lambda$, especially for dealing with distant pairs.

\subsection{MetaMixUp: Learning Data-Driven Interpolation Policy}
To solve the above problem, we propose MetaMixUp, a meta-learning based method to optimize the interpolation policy of MixUp via an online optimization. Unlike the original MixUp in Eq. \eqref{eq:mixup} which uses a predefined distribution for the interpolation coefficient $\lambda$ and a unique value in each mini-batch, we consider using different interpolation coefficients $\lambda_{i}$ in each mini-batch to improve the diversity and to make them all learnable. Our target is to tackle manifold intrusion problem through learning adaptive interpolation policy, rather than directly using the fixed interpolation coefficients of MixUp.
Specifically, the MetaMixUp is defined as: 
\begin{equation}\label{eq:metamixup}
\begin{split}
	\Tilde{x}_{i} = \lambda_{i} x_{i} + (1-\lambda_{i}) x_{j}, \\
	\Tilde{y}_{i} = \lambda_{i} y_{i} + (1-\lambda_{i})  y_{j},
\end{split}
\end{equation}
where $\lambda_i$ is optimized via meta-learning, by optimizing the meta-objective on a validation set \cite{finn2017model,RenZYU18}.

The optimal weight of a network is given by minimizing the loss function over the training set $\mathcal{D}=\{(\Tilde{x}_{i}, \Tilde{y}_{i})\}^{N}_{i=1}$:  
\begin{equation}\label{E6}
    \theta^{\ast}(\lambda) = \arg \min_{\theta} \frac{1}{N}\sum^{N}_{i=1}\ell(f(\Tilde{x}_{i};\theta),  \Tilde{y}_{i}; \lambda_{i})
\end{equation}
where $\theta$ denotes the network parameters. Analogous to architecture search using progressive evolution \cite{LiuZNSHLFYHM18}, the validation set performance is treated as fitness. The optimal $\lambda$ is then optimized on the validation set $\mathcal{D}_{v}=\{(x_{i}, y_{i})\}^{M}_{i=1}$ via
\begin{equation}
\label{E7}
\lambda^{\ast} = \arg \min_{\lambda \in [0,1]}\frac{1}{M}\sum^{M}_{i=1}\ell(f_{v}(x_{i};\theta^{\ast}(\lambda)),  y_{i})
\end{equation}
where $f_{v}$ denotes the meta network which has the same architecture as $f$. 
While $f$ is used for prediction, and $f_{v}$ is only used to optimize $\lambda_i$. Similar to other meta learning methods where parameters of a network can be quickly adapted to the task guided by  meta-objective, our meta network ($f_{v}(\theta) = f(\theta; \lambda_{i})$) 
aims to search the optimal $\lambda_{i}$ to interpolate the training samples, but using gradient descent.

Optimization of \eqref{E7} indicates a bilevel optimization problem \cite{ColsonMS07} with $\lambda$ as the upper-level variable and $\theta$ as the lower-level variable. The nested formulation also arises in gradient-based hyperparameter optimization \cite{Pedregosa16}, which is relevant in a manner that interpolation policy could be regarded as a special type of hyperparameter, although its magnitude is substantially higher than scalar-valued hyperparameters such as the learning rate, and it is harder to optimize.

Given the original MixUp examples, at $t$-th step, we first take a step forward with meta network to update model weights $\theta$: 
\begin{equation}
\label{E9}
    \theta_{t+1} = \theta_{t} - \eta\nabla_{\theta}\frac{1}{n}\sum^{n}_{i}\ell(f_{v}(\Tilde{x}_{i}; \theta),  \Tilde{y}_{i}),
\end{equation}
where $\eta$ is the learning rate. 
Then, the ideal optimal $\lambda^{*}$ can be calculated on the validation set by
\begin{equation}
\label{E6}
     \lambda^{\ast} = \arg \min_{\lambda \in [0,1]}\frac{1}{M}\sum^{M}_{i=1}\ell(f_{v}(x_{i};\theta_{t+1}),  y_{i})
\end{equation}
However, the complete optimization of the above objective exactly can be prohibitive due to the expensive inner optimization. We therefore propose a simple approximation scheme and some relaxation to make it more scalable in practice. To achieve this, we consider using an online approximation, by performing only a single gradient descent step on 
the validation set, without solving the inner optimization completely by training until convergence. 

The generalization performance of the model is measured with a validation loss based on an unregularized meta model. Hence the value of the loss depends only on elementary parameter updates \cite{LuketinaRBG16}. The gradient of the validation loss with respect to interpolation policy $\lambda$ is:
\begin{equation}
\label{E11}
    \nabla_{\lambda}\ell(f_{v},\mathcal{D}_{v}) = \frac{\partial}{\partial \lambda}\frac{1}{m}\sum^{m}_{i=1}\ell(f_{v}(x_{i};\theta_{t+1}),  y_{i}),
\end{equation}
where $m$ denotes the batch size.

To achieve our objective, we only consider the influence of the interpolation policy on the current elementary parameter update, $\nabla_{\lambda}\ell(f_{v},\mathcal{D}_{v})$. The interpolation policy update is therefore:
\begin{equation}
\label{E12}
    \lambda := \lambda -\alpha \nabla_{\lambda}\ell(f_{v},\mathcal{D}_{v}),
\end{equation}
where $\alpha$ is the step size for updating $\lambda$. Since $\lambda$ is free to span the entire set of real numbers, we project $\lambda$ back to $[0,1]$ with a sigmoid function:
\begin{equation}
\label{E13}
    \lambda^{*} := sigmoid(\lambda),
\end{equation}
The complete computation graph of MetaMixUp is summarized and visualized in Fig. \ref{fig:figure1}. 

Given the refined interpolation policy, we re-mixup the training examples and update the weight of network in training state. It is straightforward to adopt MetaMixUp in  supervised learning (SL). We detail the outline of our MetaMixUp algorithm in Algorithm \ref{alg:algorithm}. 

\begin{algorithm}[htp]
 \caption{MetaMixUp for supervised learning}
 \label{alg:algorithm}
 \textbf{Input}:
 Training data $\mathcal{D}$, validation $\mathcal{D}_{v}$. \\
 \textbf{Parameters}: Deep neural network $\Phi (\theta)$, batch size $B$, learning rate $\eta$, step size $\alpha$.\\
 \textbf{Output}: Deep neural network $\Phi(\theta)$ and $\lambda^{*}$.\\
 \begin{algorithmic}[1] 
 \FOR{$t$ = 1, 2, ... , $Iter_{max}$}
 \STATE \textbf{Random initialize} $\lambda=\{\lambda_{i}\}^{B}_{i=1}$ ;
 \STATE \textbf{Turn network to meta stage} $\Phi(\theta) \to \Phi^{'}(\theta)$;
 \STATE \textbf{MixUp} examples with $\lambda$ to construct $\Tilde{\mathcal{D}}$;\\
 \STATE \textbf{Update} $\theta^{'}=\theta-\eta\nabla_{\theta}\ell(\Phi^{'}(\theta),\Tilde{\mathcal{D}} )$;
 \STATE \textbf{Update} $\lambda^{*}=\lambda-\alpha\nabla_{\lambda}\ell(\Phi^{'}(\theta^{'}),\mathcal{D}_{v} )$;
 \STATE \textbf{MixUp} examples with updated $\lambda^{*}$ to reconstruct $\Tilde{\mathcal{D}}$;
 \STATE \textbf{Update nework} $\theta:=\theta-\eta\nabla_{\theta}\ell(\Phi(\theta),\Tilde{\mathcal{D}} )$;
 \ENDFOR
 \STATE \textbf{return} $\Phi(\theta)$ and $\lambda^{*}$
 \end{algorithmic}
 \end{algorithm}

\subsection{Extension of MetaMixUp for Semi-Supervised Learning}
It is non-trivial to apply MixUp and its variants to semi-supervised learning (SSL) tasks since MixUp performs the interpolation in both data and \emph{label} space, which is not applicable for unlabeled data. In this section, we present the strategy of adapting MetaMixUp to SSL.

Before presenting the SSL extension of MetaMixUp, we first detail the notations. We denote by $\mathcal{D} = \mathcal{L}\cup \mathcal{U}$ the entire training set with a relatively small labeled data set $\mathcal{L} = \{(x_{i}, y_{i})\vert i=1, 2, \dots, L\}$ and a large unlabeled data set $\mathcal{U} = \{(x_{i})\vert i=L+1,  ..., L+U\}$. Consider here a $C$-class classification problem, each $y_{i} = [y^{1}_{i}, y^{2}_{i}, ..., y^{C}_{i}]^{T} \in \{0,1\}^{C}$ denotes
the corresponding one-hot true label, 
such that $y^{j}_{i}=1$ if $x_{i}$ belongs to the $j$-th class and $y^{j}_{i}=0$ otherwise. 
Let $N = L + U$ be the total number of training samples, and usually we have $L \ll U$.

As the main challenge  of adapting MixUp and MetaMixUp to SSL is the lack of labels for unlabelled data, pseudo-label seems to be a natural choice of the basic SSL method for MixUp.
Typically, pseudo-labeling based methods simply leverage the high-confident predictions as the true labels and results in high performance in practice. We employ standard cross-entropy loss as classification loss. 
For a pseudo-label based SSL method, the  loss function is  
\begin{equation}
\label{E2}
    \ell(\mathcal{X}, \bar{\mathcal{Y}}; \theta) = \frac{1}{N}\sum^{N}_{i=1}\ell(f(x_{i};\theta), \bar{y}_{i}),
\end{equation}
where $\bar{\mathcal{Y}} = \{{\bar{y}_{i}}\}^{N}_{i=1}$ denotes all the true or pseudo labels for the training set $\mathcal{X} = \{{x_{i}}\}^{N}_{i=1}$. If ${x}_{i} \in \mathcal{L}$, $\bar{y}_{i}$ is fixed to its corresponding true label vector $\bar{y}_{i} = y_{i}$ throughout the entire training
process. For an unlabeled training sample $x_{i} \in \mathcal{U}$, $\bar{y}_{i}$ is the estimated label
vector by the network at current iteration. 
Utilizing pseudo-labels, we can implicitly generate the `augmented' set $\Tilde{\mathcal{D}} = \{(\Tilde{x}_{i}, \Tilde{y}_{i})\}^{N}_{i=1}$ during training process. Then unsupervised  MetaMixUp can be transformed to the supervised one.

\textbf{Asynchronous Pseudo Labeling.} 
A commonly known issue of pseudo-label based method is that the incorrect pseudo-labels of the newly labeled examples can propagate errors and thus degenerate the performance. To solve this issue, we propose an asynchronous strategy, named asynchronous pseudo labeling (APL), to improve deduced pseudo label quality for unlabeled data and thus stabilize the training. Specifically, to filter out the unconfident pseudo-labels, we predefine a threshold $\sigma$ such that all the unlabeled samples with the maximal prediction probability under $\sigma$ will be eliminated from performing back propagation of the loss function. This threshold allows to mitigate influence from the uncertain predictions of the unlabeled samples. Accordingly, unlabeled training samples are dynamically relabeled with more accurate labels since easy examples are generally predicted with high confidence while those with low confidence are more likely to be hard examples.
Instead of using a constant threshold, we asynchronously decrease the threshold to make more unlabeled data to be labeled after further epochs as shown effective in \cite{SaitoUH17}. 
The threshold $\sigma$ is decreased by $\sigma_{d}$ every $K$ epochs and is defined at $t$ epoch, by
    $\sigma_{t} = \sigma_{t-1} - \sigma_{d}\times [\frac{t}{K}]\times K$,
where $[ \cdot ]$ denotes the round-off operation. We set initial threshold $\sigma_{0}=0.95$ and $K = 30$ in all experiments to avoid the frequent update of pseudo-label. 
The framework is summarized in Algorithm \ref{alg:semialgorithm}. 

\begin{algorithm}[htp]
\caption{MetaMixUp for Semi-Supervised Learning}
\label{alg:semialgorithm}
\textbf{Input}:
Labeled training data $\mathcal{D}_{L}$, unlabeled training data $\mathcal{D}_{UL}$, validation $\mathcal{D}_{v}$,  \\
\textbf{Parameters}: Deep neural network $\Phi(\theta)$, batch size $B$, learning rate $\eta$, step size $\alpha$, pseudo-label threshold $\sigma$, decrease step $\sigma_{d}$, maximum iterations $Iter_{max}$.\\
\textbf{Output}: Deep neural network $\Phi(\theta)$ and $\lambda^{*}$ \\
\begin{algorithmic}[1] 
\FOR{$t$ = 1, 2, ... , $Iter_{max}$}
\STATE \textbf{Random initialize} $\lambda=\{\lambda_{i}\}^{2B}_{i=1}$;

\STATE {\bf Labeling} $\bar{y}_{j}=\arg\max \Phi(x_{j},\theta)$ \textbf{if} $\Phi(x_{j},\theta)>\sigma$
\IF{$t$ reach update period}
\STATE $\sigma = \sigma - \sigma_{d}$
\ENDIF
\STATE \textbf{Turn network to meta stage} $\Phi(\theta) \to \Phi^{'}(\theta)$;
\STATE \textbf{MixUp} examples with $\lambda$ to construct $\Tilde{\mathcal{D}}_{L}$ and $\Tilde{\mathcal{D}}_{UL}$;\\
\STATE {\bf Calculate MetaLoss} \\
$MetaLoss = \bar{L}_{s}(\Phi^{'}(\theta),\Tilde{\mathcal{D}}_{L}) + \bar{L}_{us}(\Phi^{'}(\theta),\Tilde{\mathcal{D}}_{UL})$

\STATE \textbf{Update} $\theta^{'}=\theta-\eta\nabla_{\theta}MetaLoss$;
\STATE \textbf{Update} $\lambda^{*}=\lambda-\alpha\nabla_{\lambda}\ell(\Phi^{'}(\theta^{'}),\mathcal{D}_{v} )$;
\STATE \textbf{MixUp} examples with $\lambda^{*}$ to reconstruct $\Tilde{\mathcal{D}}_{L}$ and $\Tilde{\mathcal{D}}_{UL}$;\\
\STATE {\bf Calculate Loss} \\
$Loss = L_{s}(\Phi(\theta),\Tilde{\mathcal{D}}_{L}) + L_{us}(\Phi(\theta),\Tilde{\mathcal{D}}_{UL})$ 
\STATE \textbf{Update} $\theta:=\theta-\eta\nabla_{\theta}Loss$;
\ENDFOR
\STATE \textbf{return} $\Phi(\theta)$ and $\lambda^{*}$
\end{algorithmic}
\end{algorithm}

\section{Experiments}
We study here the regularization properties of our MetaMixUp on typical image classification benchmarks. The aim of the following experiments is threefold. First, we investigate the impact of the vanilla interpolation policy of original MixUp on the quality of the solution on multiclass classification problems. Second, we test our proposed MetaMixUp methods in context of supervised learning and semi-supervised learning. Finally, we constrast the MetaMixUp technique against classical MixUp approaches to learn models with better regularization properties.

\subsection{Datasets}
\textbf{MNIST and Fashion.}
Both datasets contain 60,000 training and 10,000 test images (28$\times$28) of 10 classes.  Both datasets are  used for \emph{supervised learning} under the standard   training and test splits.

\textbf{CIFAR-10 and CIFAR-100.}
{CIFAR-10 and CIFAR-100 have 10 and 100  classes of natural  images ($32\times 32$) respectively.}
For \emph{supervised learning}, we use the standard data split for training (50000) and test (10000). 
For \emph{semi-supervised learning}, we follow  \cite{OliverORCG18}, where 1000 or 4000 images (100 or 400  per class) in CIFAR-10 are selected from the training set as the labeled training examples, and the remaining  as the unlabeled  data. CIFAR-100 is only used for \emph{supervised learning task}.

\textbf{SVHN.}
The Street View House Numbers (SVHN) dataset \cite{netzer2011reading} contains real world $32\times32$ images of house numbers.
It contains 73,257 training and 26,032 test images. We use the standard training/test split for \emph{supervised learning}. For \emph{semi-supervised learning}, we follow  \cite{OliverORCG18} to  randomly select 50 or 100  samples per class from the training set as the labeled data, the remaining  as the unlabeled data.

\textbf{ImageNet.}
ImageNet-2012 is a dataset for classification \cite{RussakovskyDSKS15} with 1.3 million training images, 50,000 testing images, and 1,000 classes. We used the standard data split for supervised training, and followed the data processing approaches used in AdaMixUp \cite{Adamixup}, in which the crop size is 100$\times$100 instead of 224$\times$224 due to limited computational resources.

\subsection{Implementation Details}

{To make a fair comparison and follow the settings of existing works, we train all the models from scratch. 
We set step size $\alpha$ to 5.0 and choose SGD as the optimizer for all the experiments with the momentum
 0.9 and weight decay  $10^{-4}$. We train our models on two NVIDIA GTX 1080 Ti GPUs.} 

\textbf{Supervised Learning (SL)}
Following \cite{Adamixup}, a 3-layer CNN is used for the tasks on MNIST and Fashion. For CIFAR-10, CIFAR-100, SVHN and ImageNet,  the PreAct-ResNet-18 \cite{HeZRS16eccv}, PreAct-ResNet-34 and Wide-ResNet-28-10 architectures are used for all the considered methods. 
We set the batch size as 128, initial learning rate as 0.1 followed by cosine annealing \cite{LoshchilovH17}, number of epochs as 600.
{We randomly sample 1000 images (100  per class) from the original training set to construct our meta validation set.}
For data augmentation, we only use  horizontal flipping following \cite{zhangmixup} for all of the training datasets.
 
\textbf{Semi-Supervised Learning (SSL)}
We follow the unified evaluation platform  \cite{OliverORCG18}   for SSL to make fair comparisons. We set batch size 100 for both CIFAR-10 and SVHN. All the models are trained with 200 epochs using data augmentation (horizontal flip and 2-pixel translation) following \cite{OliverORCG18}. {We select 500 images (50  per class) from the  training set as  our meta validation set}. For asynchronous pseudo-labeling threshold decreasing step size, we set $\sigma_{d}=0.05$.  Following \cite{ZagoruykoK16}, learning rate starts from 0.1, and is divided by 10 after 60, 120 and 180 epochs respectively. For a fair comparison, we run all methods with WideResNet-28-2 \cite{ZagoruykoK16}  and the average  accuracy is obtained by 5 runs.

\subsection{Results}
\textbf{Results of Supervised Learning.} We compare MetaMixUp with the baseline method (w/o MixUp) and two counterpart methods (MixUp and AdaMixUp) on five popular supervised databases. We train a variety of residual network for each method. The error rates presented in Table \ref{tab:sup} show that our MetaMixUp substantially outperforms the baseline (w/o MixUp), MixUp and AdaMixUp on all the five datasets.

Surprisingly, training with MixUp does not always improve the performance. For instance, baseline method (w/o MixUp) excels MixUp (0.54\% vs. 0.59\% on MNIST and 52.84\% vs. 55.06\% on ImageNet). The error rate rise of MixUp over baseline (w/o MixUp) on ImageNet is more distinct, which is also observed on PreActResNet34 and Wide-ResNet-28-10 architectures (3.01\% and 2.55\% top-1 error rate rise respectively). This issue strongly indicates that MixUp is not rubust to all datasets or tasks with its default interpolation policy, and a smart refinement method is needed.

As described in \cite{zhangmixup}, interpolation policy $\lambda$ of vanilla MixUp is drawn according to a Beta distribution: $\lambda \sim \beta(\alpha, \alpha)$, where $\alpha$ is an extra hyperparameter needed to turn.  With $\alpha = 1.0$, this is equivalent to sampling from an uniform distribution $U(0, 1)$. To investigate the impact of $\alpha$ and $\lambda$ on the performance, we operate hyperparameter tuning on $\alpha$ which determines the distribution of $\lambda$. The influence is presented by experimental error rate on three datasets in Table \ref{tab:hyper}. We find that for $\alpha$, vanilla MixUp has different behaviors on disparate datasets, i.e., original MixUp somehow increases the error rate on MNIST and SVHN compared to baseline. This on the contrary suggests that mixing samples in a data-adaptive way tends to provide positive impacts on generalization performance. To further study how interpolation policy impacts the generalization performance, we directly fix $\lambda$ for each pair of training samples in the training process, in this scenario, the performance becomes more unstable even worse than training without MixUp. The findings above substantially supports our goal to directly make optimization on $\lambda$ to generate data-adaptive interpolation policy for better MixUp techniques. As we see, the proposed MetaMixUp significantly improves vanilla MixUp without need of turning interpolation policy distribution. 

Our MetaMixUp improves vanilla MixUp by a large margin on easier  tasks (0.38\% vs. 0.59\% on MNIST and 5.15\% vs. 7.31\% on Fashion) as well as on harder tasks, SVHN (2.96\% vs. 3.83\%), CIFAR-10 (3.12\% vs. 4.57\%), CIFAR-100 (20.36\% vs. 21.35\%) with PreActResNet18. Compared with the original MixUp, the promotion is more obvious while outperforming the original version by 7.71\% top-1 accuracy on ImageNet. This improvement is consistent at different architectures, where MetaMixUp (top-1 error 47.35\%, top-5 error 24.43\%) outperforms the original MixUp (top-1 error 55.06\%, top-5 error 31.32\%) with PreActResNet34 and (top-1 error 46.38\%, top-5 error 23.91\% versus top-1 error 52.96\%, top-5 error 29.28\%) with Wide-Resnet28-10. These comparisons suggest the significant of learning a suitable interpolation policy.
As a competitive counterpart of MetaMixUp, AdaMixUp has a generator that outputs the interpolation policy and a discriminator that judges the quality of the generated policy. 
As Table \ref{tab:sup} shows, MetaMixUp distinctly outperforms AdaMixup which is strengthened with a discriminator (0.38\% vs. 0.49\%) on MNIST, (5.15\% vs. 6.21\%) on Fashion, (2.96\% vs. 3.12\%) on SVHN, (3.12\% vs. 3.52\%) on CIFAR-10, and (20.36\% vs. 20.97\%), (47.55\% vs. 49.17\%) on the more challenging CIFAR-100 and ImageNet benchmark respectively. Moreover, experiments on three different architectures consistently demonstrate the remarkable improvement gained by MetaMixUp, and the best results are achieved on Wide-ResNet-28-10. These comparisons firmly demonstrate the effectiveness of MetaMixUp.

\begin{table*}[!htb]
\centering
\caption{Error rates $(\%)$ of supervised learning on test set. $\ddag$ refers to the results from \cite{Adamixup}.}
\resizebox{0.9\textwidth}{!}{
\begin{tabular}{l|cc|ccccc}  
\hline
\multirow{2}{*}{Datasets}              & \multirow{2}{*}{MNIST} & \multirow{2}{*}{Fashion}& \multirow{2}{*}{SVHN}& \multirow{2}{*}{CIFAR-10}   & \multirow{2}{*}{CIFAR-100} & \multicolumn{2}{c}{ImageNet}  \\
\cline{7-8}
& & & & & & Top-1&Top-5\\
                   
\hline
 Architecture&    \multicolumn{2}{c|}{3-layer CNN }   &  \multicolumn{5}{c}{PreActResNet18}      \\
\hline
Baseline    &0.54 & 7.31 & 4.62& 5.62       & 25.20 &  52.84 & 29.48    \\
MixUp \cite{zhangmixup}  &0.59 & 6.74& 3.83 &4.57   & 21.35 &   55.06 & 31.32  \\
AdaMixup w/o Discriminator $\ddag$ \cite{Adamixup}                   & - & -& -&3.83   & 24.75 &  - & -    \\
AdaMixup w Discriminator $\ddag$ \cite{Adamixup}      & 0.49 & 6.21 &3.12 & 3.52&   20.97 &   49.17 & 25.78   \\
{MetaMixUp (ours)}    &\textbf{0.38} & \textbf{5.15}&\textbf{2.96} & \textbf{3.12}    &    \textbf{20.36} & \textbf{47.35}& \textbf{24.43} \\
\hline
Architecture &    &  &   \multicolumn{5}{c}{PreActResNet34}      \\
\hline
Baseline    &- & - & 4.46& 5.32       & 24.63 &  50.25 & 28.59    \\
MixUp \cite{zhangmixup}  &- & -& 3.35 & 4.14   & 20.85 &   53.26 & 29.83  \\
{MetaMixUp (ours)}    &- & - &\textbf{2.42} & \textbf{2.51}    &    \textbf{18.51} & \textbf{46.89}& \textbf{23.93} \\
\hline
Architecture  &     &  &   \multicolumn{5}{c}{Wide-Resnet-28-10}      \\
\hline
Baseline    &- & - & 4.34& 4.74       & 22.33 &  50.41 & 28.38    \\
MixUp \cite{zhangmixup}  &- & -& 3.31 &3.07   & 19.21 &   52.96 & 29.28  \\
{MetaMixUp (ours)}    &- & - &\textbf{2.33} & \textbf{2.48}    &    \textbf{18.45} & \textbf{46.38}& \textbf{23.91} \\
\hline
\end{tabular}}

\label{tab:sup}
\end{table*}

\begin{table*}[!htb]
\centering
\caption{Test error rates ($\%$) of supervised learning for different $\alpha$ and $\lambda$ on test set of CIFAR-10, CIFAR-100, SVHN and MNIST. Note that $\alpha=0$ indicates standard training without MixUp. 
}
\resizebox{0.6\textwidth}{!}{
\begin{tabular}{lcccc}  
\hline
Method              & CIFAR-10  & CIFAR-100 & SVHN  & MNIST     \\
                                                                \\
\hline
MixUp ($\alpha=0$)  & 5.62      & 25.20     & 4.62  & \bf0.54      \\
MixUp ($\alpha=0.5$)& 4.65      & 21.67     & 4.71  & 0.66      \\
MixUp ($\alpha=1$)  & \bf4.57      & \bf21.35     & 4.03  & 0.59      \\
MixUp ($\alpha=2$)  & 4.78      & 21.49     & \bf3.86  & 0.61      \\
\hline
MixUp ($\lambda=0.1$) & 5.02    & 22.56   & 4.21   & \bf0.56       \\
MixUp ($\lambda=0.2$) & \bf4.69 & 22.83   &\bf4.14 & 0.59       \\
MixUp ($\lambda=0.3$) & 4.71    &\bf22.44 & 4.42   & 0.64       \\
MixUp ($\lambda=0.4$) & 5.29    & 23.29   & 4.73   & 0.71       \\
MixUp ($\lambda=0.5$) & 5.88    & 23.62   & 4.78   & 0.70       \\
\hline
MetaMixUp (ours)    &{\bf2.48}  &{\bf18.45} &{\bf2.33}  &{\bf0.38}  \\
\hline
\end{tabular}}

\label{tab:hyper}
\end{table*}

\begin{table*}[!htb]
\centering
\caption{Test error rates ($\%$) of SSL approaches on test set of  CIFAR-10 (1K  means 1K labelled data) and SVHN. 'Supervised-Only' refers to no unlabeled data. $\dag$ refers to the results reported in \cite{OliverORCG18}. 
}
\resizebox{0.8\textwidth}{!}{
\begin{tabular}{lcccc}  
\hline
Method      & CIFAR-10        & CIFAR-10   & SVHN        & SVHN    \\
            & 1K Labels       & 4K Labels  & 500 Labels  & 1K Labels\\
\hline
Supervised-Only   & 35.95 $\pm$  0.25      & 20.34 $\pm$  0.33  & 17.71 $\pm$  0.33 & 12.93 $\pm$ 0.37      \\
Pseudo-Label      & 25.13 $\pm$  0.46    & 17.73 $\pm$ 0.55   & 9.91 $\pm$  0.37  & 7.82 $\pm$ 0.31       \\
$\Pi$-Model $\dag$ \cite{LaineA16}  & 24.81 $\pm$  0.51       & 16.37 $\pm$ 0.63  & 8.82 $\pm$ 0.20  & 7.19 $\pm$ 0.27      \\
Mean Teacher $\dag$ \cite{TarvainenV17} & 23.38 $\pm$  0.24    & 15.87 $\pm$ 0.28  & 8.01 $\pm$  0.22  & 5.65 $\pm$ 0.47      \\
VAT $\dag$ \cite{miyato2018virtual}     & 21.52 $\pm$  0.31    & 13.86 $\pm$ 0.27  & 6.21 $\pm$ 0.46  & 5.63  $\pm$ 0.20      \\
VAT + EM $\dag$ \cite{miyato2018virtual} & 21.28 $\pm$  0.37   & 13.13 $\pm$ 0.39  & 6.14 $\pm$  0.41  & 5.35 $\pm$ 0.19 \\
\hline
MixUp + Pseudo-Label     & 23.47 $\pm$  0.45     & 15.04 $\pm$ 0.33  & 7.42 $\pm$ 0.44   &    6.42 $\pm$ 0.25\\
MixUp + APL              & 22.85 $\pm$  0.39     & 14.79 $\pm$ 0.22  & 6.73  $\pm$ 0.31   &    6.37 $\pm$ 0.37\\
MetaMixUp + Pseudo-Label (ours)  & 21.29 $\pm$  0.31    & 12.64 $\pm$ 0.45  & 6.18 $\pm$ 0.46   &     6.17 $\pm$ 0.35\\
MetaMixUp + APL (ours)   & {\bf 20.66 $\pm$ 0.24} & {\bf 11.50 $\pm$ 0.22}   & {\bf6.05 $\pm$ 0.43}     &     {\bf5.34 $\pm$ 0.31}\\
\hline
MixMatch w/o MixUp &- &  10.97 $\pm$ 0.34&- & 4.89 $\pm$ 0.41\\
MixMatch          & 7.87 $\pm$ 0.36 & 6.42 $\pm$ 0.21& 3.79 $\pm$ 0.65&3.32 $\pm$ 0.43\\
MixMatch + MetaMixUp (ours) &{\bf7.69 $\pm$ 0.29} &{\bf 6.21 $\pm$ 0.24} &{\bf 3.63 $\pm$ 0.45} &{\bf 3.25 $\pm$ 0.41} \\
\hline
\end{tabular}}

\label{tab:semi}
\end{table*}

\textbf{Results of Semi-Supervised Learning.}
Using a standard unified and fair SSL evaluation framework~\cite{OliverORCG18}, we compare our pseudo-labeling extension of MixUp and MetaMixUp with the state-of-the-art methods on CIFAR-10 and SVHN benchmarks. The  error rates are presented in Table \ref{tab:semi}.
The best performance of pseudo-label based methods is achieved by MetaMixUp combining with asynchronous pseudo labeling (MetaMixUP+APL) introduced in Section 3.4. The improvement over its purely pseudo-labeling counterpart is about 6.2\% on CIFAR-10 (4K labels) and 2.5\% on SVHN (1K labels) in terms of accuracy.  When comparing on fewer label settings, the best performance of our MetaMixUP+APL reports in error rate 20.66\% on CIFAR-10 (1K labels) and 6.05\% on SVHN (500 labels). More importantly, all the pseudo-based MetaMixUp outperform its MixUp counterparts, showing the effectiveness of the interpolation policies learned by MetaMixUp.

\begin{table}[!h]
\centering
\caption{Test error rates ($\%$) obtained with a unified  implementation of  various SSL methods.  
}
\resizebox{0.5\textwidth}{!}{
\begin{tabular}{lcccc}  
\hline
Method              & CIFAR-10& CIFAR-10   & SVHN& SVHN    \\
                    &1K Labels& 4K Labels  & 500 Labels& 1K Labels\\
\hline
Pseudo-Label           & 25.13 $\pm$  0.46 & 17.73 $\pm$ 0.55  &  9.91 $\pm$  0.37& 7.82 $\pm$ 0.31       \\
APL     &\textbf{24.36 $\pm$  0.62} & \textbf{16.96 $\pm$ 0.32}  & \textbf{8.68 $\pm$  0.72}&  \textbf{7.46 $\pm$ 0.27}\\

\hline
\end{tabular}
}
\label{tab:apl}
\end{table}
We also notice that using asynchronous pseudo labeling (APL) for both MixUp and MetaMixUp provides a slight improvement benefit. To illustrate the capability of correctly updating pseudo labels for unlabeled data, we also perform a specialized comparison between APL and the baseline pseudo-labeling method (see Table~\ref{tab:apl}). In contrast, APL gradually utilizes the more reliable and stable pseudo labels to enforce classification loss, and hence outperforms pseudo-labeling method \cite{lee2013pseudo} on both CIFAR-10 and SVHN datasets without additional computation cost.

As a complimentary experiment to further provide evidence about improvement obtained by MetaMixUp, we replaced MixUp with our MetaMixUp on MixMatch \cite{mixmatch} to generate interpolation policy for data augmentation. We find that the error rates of \textit{MixMatch + MetaMixUp} are consistently lower than \textit{MixMatch} over all datasets, surpassing the published state-of-the-art approaches by a significant margin to our best knowledge. On dataset CIFAR-10, we achieved test error of 7.69\% (MixMatch + MetaMixUp) with 1K labels and 6.21\% with 4K labels. We achieved an error rate of 3.63\% with only 500 labels on SVHN compared to MixMatch performance of 3.79\%, showing the great success of the adaptation of MetaMixUp to semi-supervised learning.


\begin{figure}[htp!]
    \centering  
	\includegraphics[width=1.0\linewidth, trim=250 95 250 110,clip]{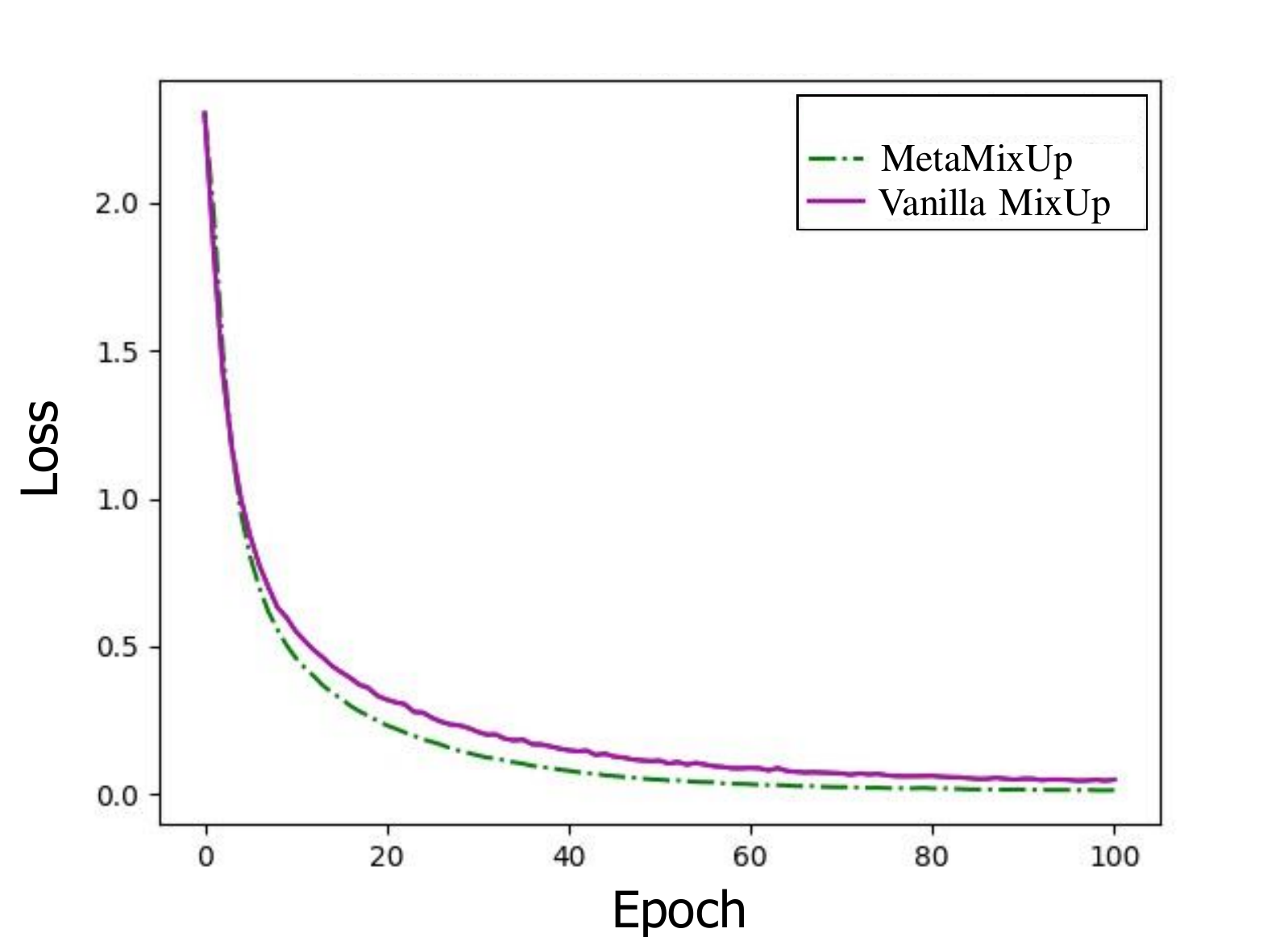}
    \caption{The loss of the supervised training with MetaMixUp and MixUp on CIFAR-10.} 
	\label{fig:loss}
\end{figure}

\subsection{Further Discussion}
In this section, we experimentally study some properties of MetaMixUp and try to answer the following questions: (i) how sensitive is MetaMixUp to the validation size; (ii) does MetaMixUp really mitigate the manifold intrusion issue; (iii) how does the relative frequency of the learned interpolation weighting coefficient differ from the Beta distribution used in MixUp \cite{zhangmixup} on MNIST; (iv) how does the distribution of the learned interpolation policy change during training on CIFAR-10 and (v) how sensitive is MetaMixUp with APL to the confidence threshold hyperparameter in SSL tasks.

\subsubsection{Trade-off of validation size}
Validation size controls the quantity of examples of each class for meta-stage. In order to make a trade-off and understand the sensitivity to the size of validation. Figure \ref{fig:validationsize} plots the classification error rate when we vary the size of validation set on MetaMixUp under supervised configuration. Surprisingly, using 10 examples for each class only results in 0.25\% drop on CIFAR-10 and 0.1\% on SVHN. In contrast, we observe that overall error rate does not drop when having more than 100 examples for each class of validation set. It suggests that our method dose not rely on a larger size of validation set for better performance.

\begin{figure}[!htp]
\centering  
\includegraphics[width=1.0\linewidth]{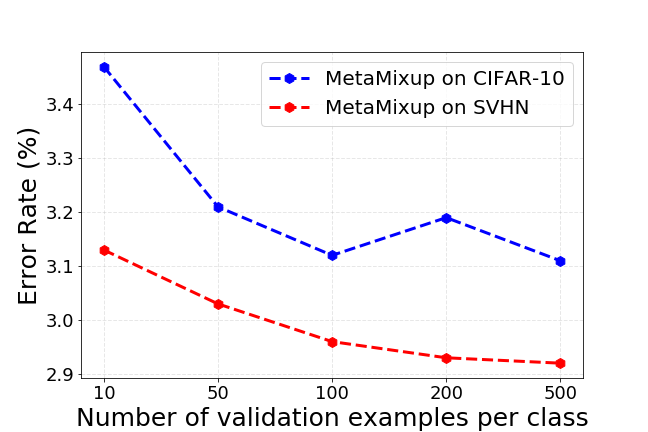}
\caption{Error rate for different validation size of MetaMixUp on CIFAR-10 and SVHN under supervised configure.}
\label{fig:validationsize}
\end{figure}

\subsubsection{Alleviation of manifold intrusion issue}
Manifold intrusion \cite{Adamixup} means improper selection of MixUp weight $\lambda$ can cause the conflicts of the labels of original data and those of mixuped data, leading to degraded performance. 
To verify the effectiveness of our method on mitigating manifold intrusion,
 we train a simple CNN network on MNIST with MetaMixUp under supervised setting and extract the features of all the mixed images during training. For each mixed sample, we compute the minimal Euclidean distance in the feature space of the trained network between this mixed sample and all the different labeled samples. Then, we compute the average and the minimum of these distances respectively for MixUp and MetaMixUp. The results are reported in Figure~\ref{fig:figure4}. The generated samples for MetaMixUp turn out to be farther to the set of the existing real samples. On the other hand, we also plot in Figure~\ref{fig:figure2} the features learned by a network with 2 hidden dimensions. The representation learned with MetaMixUp is more discriminative and thus the collision in the feature occurs with lower probability. Both of the above experiments confirm that MetaMixUp mitigates manifold intrusion to great extent. 
 
Figure \ref{fig:loss} plots the training loss curve of vanilla MixUp and MetaMixUp under a representative setting: ResNet-50 on CIFAR-10, where the $x$-axis denotes the training epochs. The $y$-axis is the training loss on training data. The figure shows two insights. First, the training error of MetaMixUp approaches zero. This empirically verifies the convergence of the model. Second, the loss curve in Figure \ref{fig:loss} generally satisfy the condition in Proposition 1, i.e.,  minimizing  MixUp loss helps controlling the Lipschitz constant. The training loss of MetaMixUp is substantially lower than vanilla MixUp. It suggests that the proposed meta-learning schema to learn mixing policies alleviates the manifold intrusion/underfitting issue in vanilla Mixup and thus optimizes an underlying robust objective.
 
 \begin{figure}[htp]
	\centering  
	\includegraphics[width=1.0\linewidth]{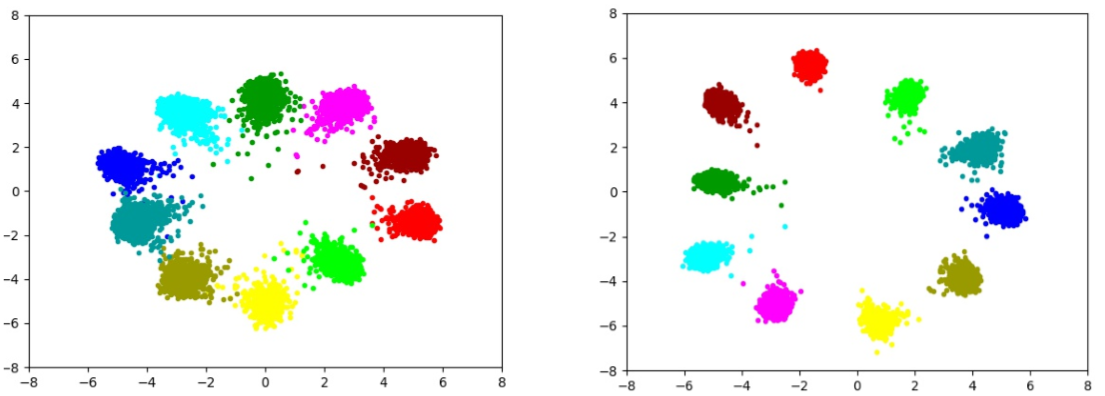}
	\caption{{Feature  of MixUp (left) and MetaMixUp (right) on MNIST.} } 
	\label{fig:figure2}
\end{figure}

 \begin{figure}[htp]
	\centering  
	\includegraphics[width=1.0\linewidth]{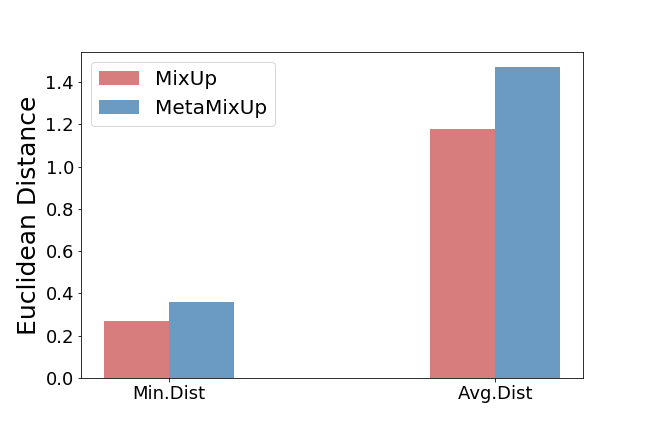}
	\caption{The distance between mixed samples and different labeled original examples. } 
	\label{fig:figure4}
\end{figure}


\subsubsection{Relative frequency of interpolation coefficient on MNIST}
The relative frequency of the interpolation coefficient $\lambda$ generated via MetaMixUp and vanilla MixUp of $\alpha=1.0$ is shown in Figure~\ref{fig:figure3}, when trained on MNIST. A higher frequency of 0 and 1 has occurred in MetaMixUp compared to MixUp. Since manifold intrusion is more likely to occur on MNIST. A possible explanation of this observation is: when mixed samples may cause useless or negative effects, MetaMixUp tries to maintain the original examples to avoid possible collisions or unexpected performance degradation and thus mitigate underfitting.

\begin{figure}[htp!]
    \centering  
	\includegraphics[width=1.0\linewidth,height=0.4\linewidth]{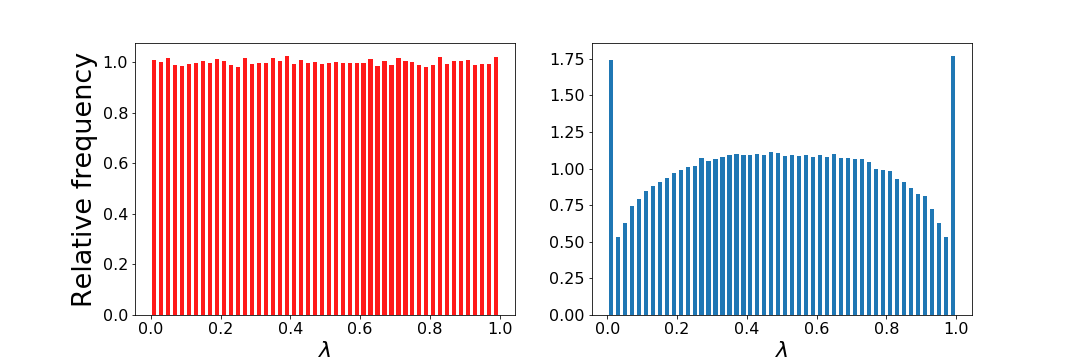}
    \caption{Relative frequency of $\lambda$ generated via MixUp (red) and MetaMixUp (blue) on MNIST.} 
	\label{fig:figure3}
\end{figure}

\begin{figure}[htp!]
    \centering  
    \subfigure[epoch = 3]{
	\includegraphics[width=0.46\linewidth]{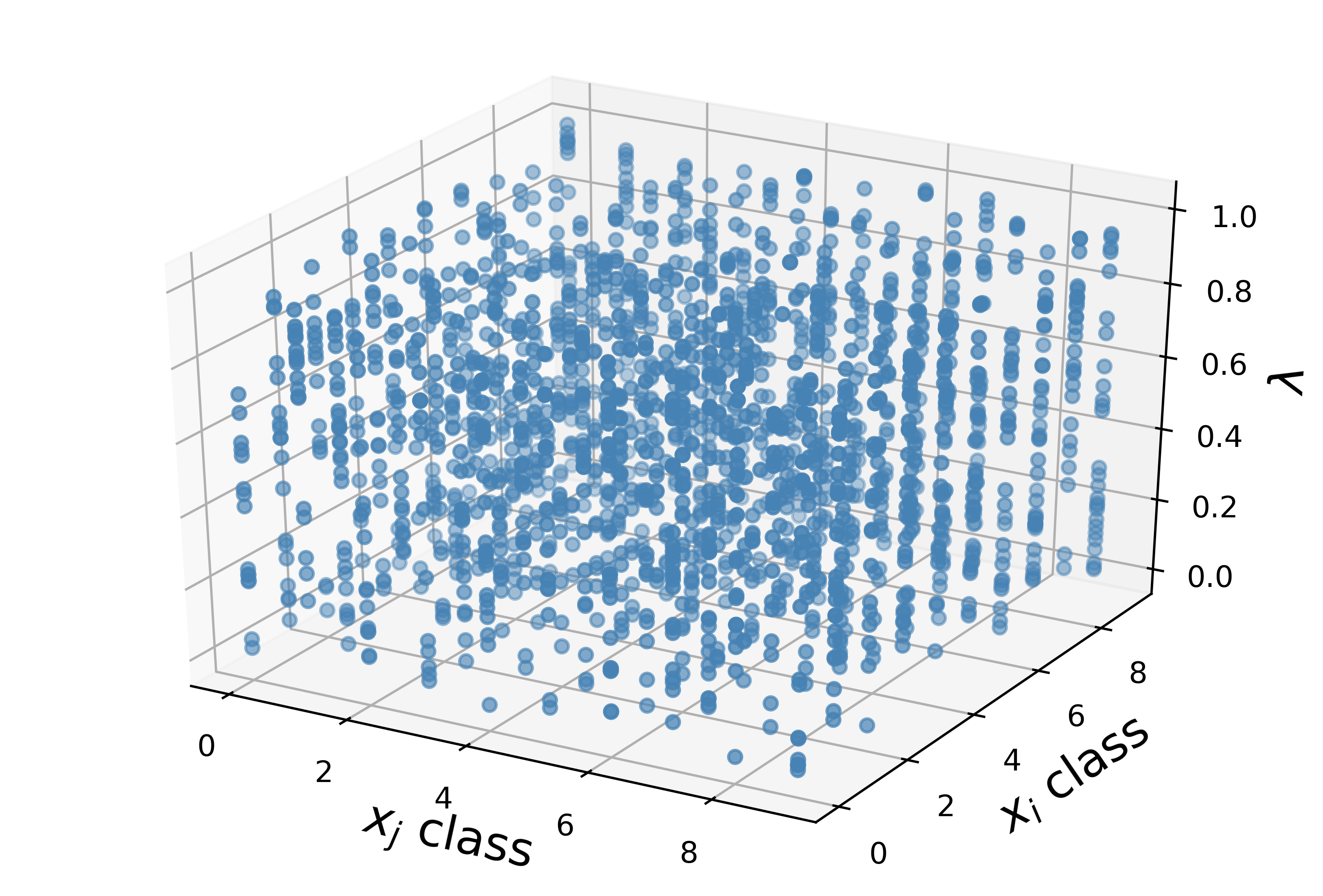}
	}
	\subfigure[epoch = 30]{
	\includegraphics[width=0.46\linewidth]{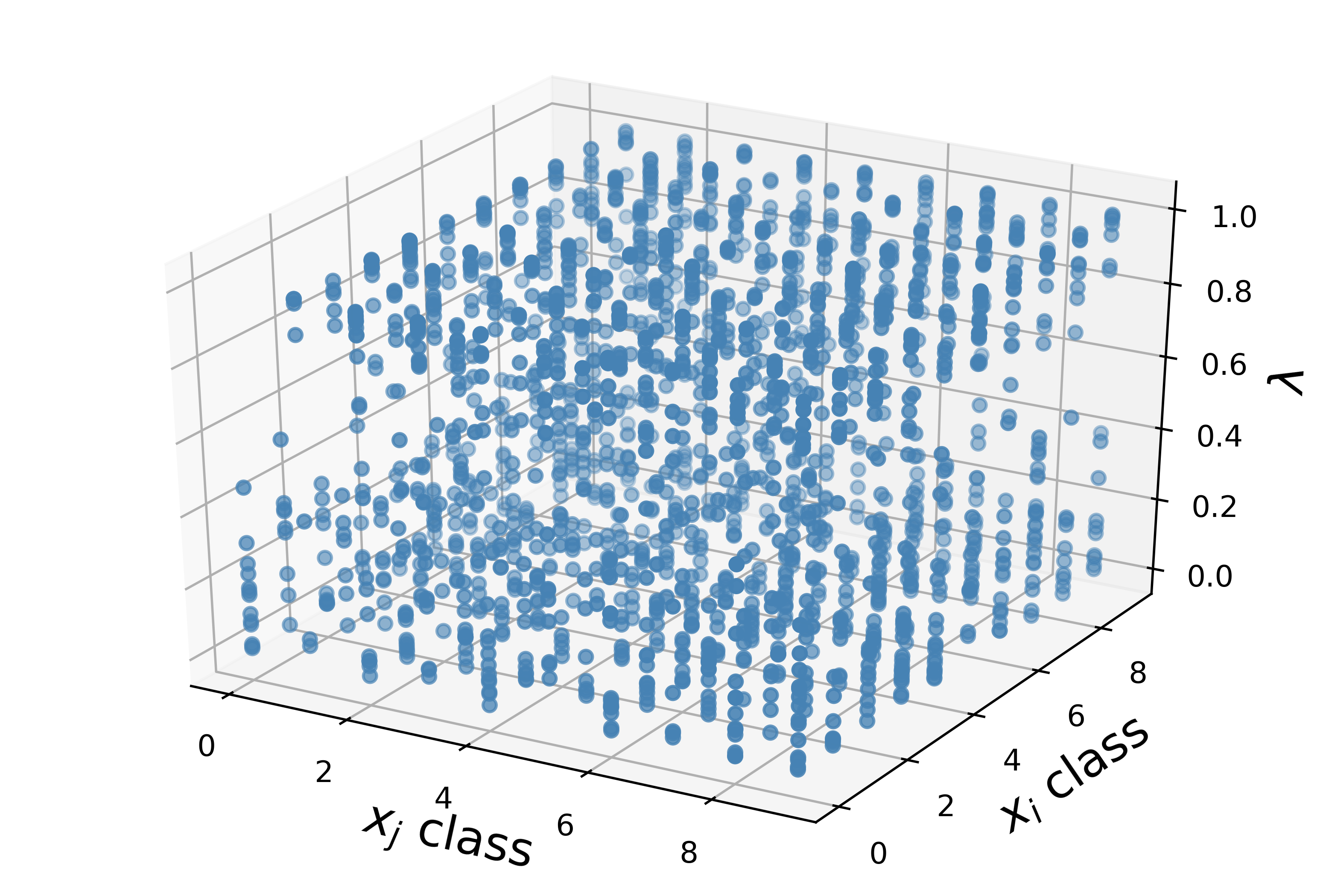}
	}
	\quad
	\subfigure[epoch = 90]{
	\includegraphics[width=0.46\linewidth]{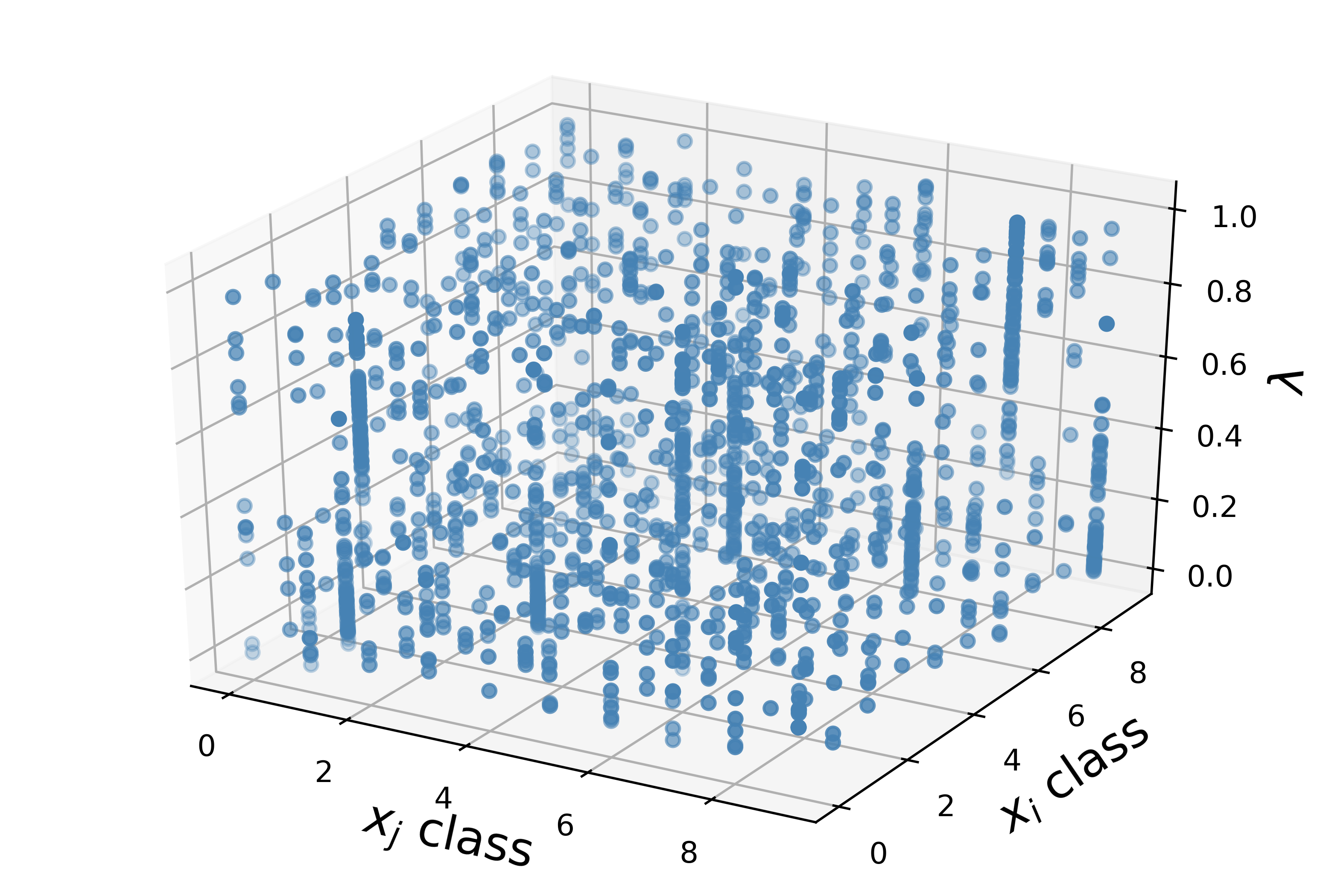}
	}
	\subfigure[epoch = 120]{
	\includegraphics[width=0.46\linewidth]{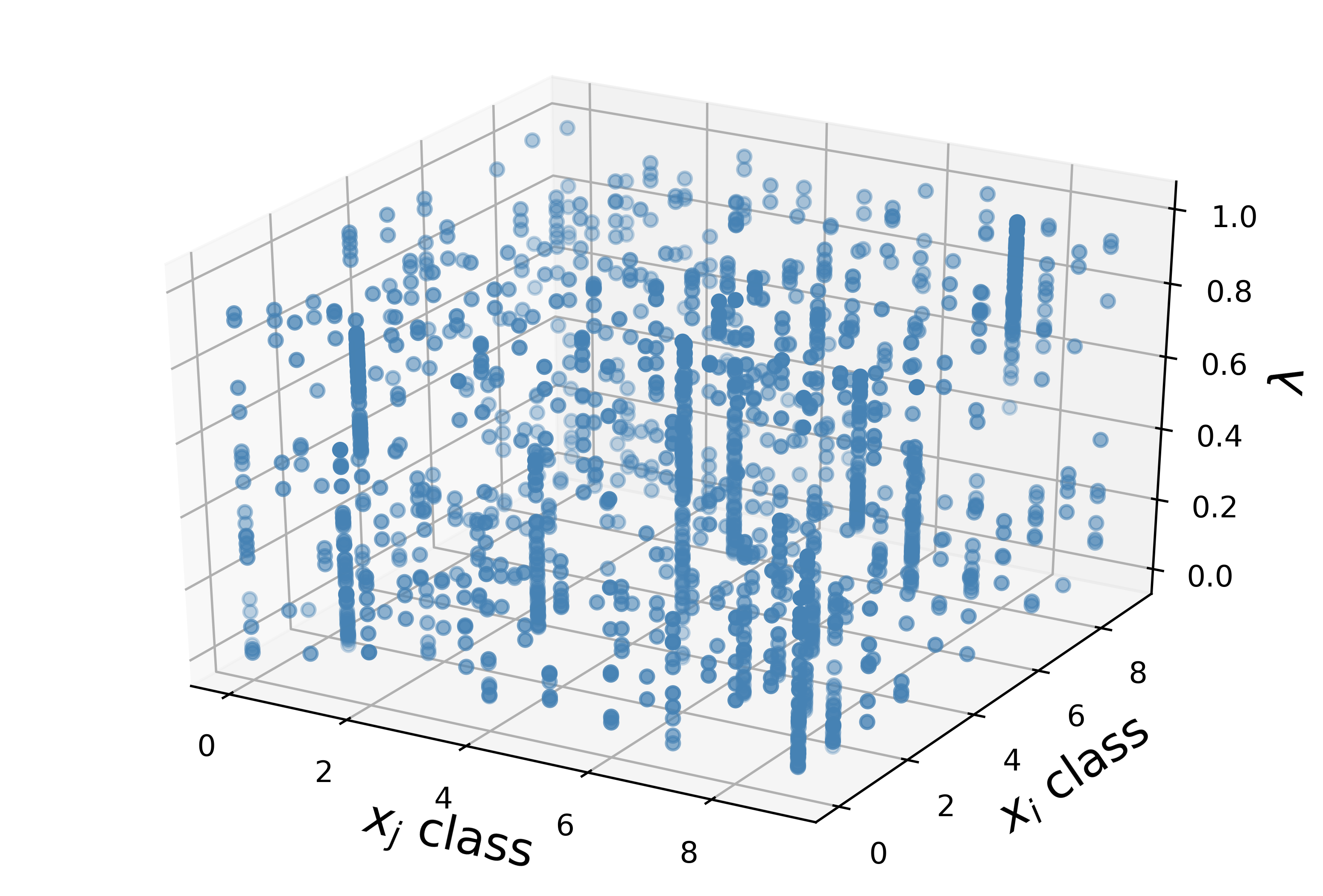}
	}
    \caption{The data-driven interpolation policy distribution learned by MetaMixUp with Resnet-50 on CIFAR-10 at epoch 3 in (a), epoch 30 in (b), epoch 90 in (c) and epoch 120 in (d).} 
	\label{fig:distribute}
\end{figure}

\subsubsection{Distribution of learned interpolation coefficient on CIFAR-10}
Figure \ref{fig:distribute} visualizes the distribution of generated interpolation policy during training on CIFAR-10 in our experiments, where $z$-axix denotes the $\lambda$ learned by our MetaMixUp; the $y$ and $x$ axes denote the class of training samples $x_{i}$ and $x_{j}$. Two observations can be found in Figure \ref{fig:distribute}. First, the learned interpolation policy changes during the training process. Figure \ref{fig:distribute} (a), (b), (c) and (d) are distribution of $\lambda$ learned at different epochs. As shown, the interpolation policy is learned from a random form to a data specific distribution. Second, the learned $\lambda$ gradually converge in idiographic sections for samples of different classes. It satisfy our intuition that the proposed method learns data-driven interpolation policy for MixUp technique.

\subsubsection{Sensitivity to $\sigma_d$}
As introduced in Section 3.4, $\sigma_{d}$ controls the decay rate of the confidence threshold for pseudo-labeling in SSL.
We conduct experiments with $\sigma$ varying from 0.01 to 0.2 to understand the sensitivity to this hyperparameter. The results (see Table~\ref{tab:threshold}) shows that influence of changing $\sigma_{d}$  on performance is small and MetaMixUp is not sensitive to $\sigma_{d}$.

\begin{table}[htp!]
\centering
\caption{Test error rate $(\%)$ for MetaMixUp+APL with different $\sigma_{d}$}
\begin{tabular}{ccccc}  
\hline
$\sigma_{d}$  & 0.01  & 0.05  & 0.1 & 0.2 \\
\hline
CIFAR-10      & 11.84   & 11.50 &  11.22  &  11.98 \\
SVHN          & 5.94   & 5.34   &   5.55   &  5.67   \\
\hline
\end{tabular}
\vspace{1mm}
\label{tab:threshold}
\end{table}

\section{Conclusions and Future Work}
In this paper, we show that vanilla MixUp loss is a lower bound of the Lipschitz constant of the gradient of the classifier function. If not smartly choosing the interpolation coefficient for each pair samples, the model suffers from the underfitting problem leading to a degradation of performance. The proposed MetaMixUp addresses this problem by optimize the interpolation policy of MixUp method with meta-learning scheme in an online fashion. The interpolation policy of MetaMixUp is learned data-adaptively to improve the generalization performance of the model. Experimental results illustrate that MetaMixUp is adaptive to supervised and semi-supervised learning scenarios with remarkable performance improvement over original MixUp and its variants. Our proposed methods achieve competitive performance across multiple supervised and semi-supervised benchmarks. In the future, it would be more interesting to explore the power of MetaMixUp in other challenging tasks. We believe that application-specific adaption of the MetaMixUp of the training objective and optimization trajectories will further improve results over a wide range of application specific areas, including domain adaption \cite{FrenchMF18}, generative adversarial networks, or semi-supervised natural language processing.
\vspace{5mm}

\ifCLASSOPTIONcaptionsoff
  \newpage
\fi



%



\bibliographystyle{IEEEtran}
\bibliography{ref}

\end{document}